\documentclass{article}

\usepackage[preprint]{neurips_2024}

\usepackage[utf8]{inputenc} 
\usepackage[T1]{fontenc}    
\usepackage{hyperref}       
\usepackage{url}            
\usepackage{booktabs}       
\usepackage{amsfonts}       
\usepackage{nicefrac}       
\usepackage{microtype}      
\usepackage{xcolor}         

\usepackage{subcaption}
\usepackage{graphicx}
\usepackage{amsfonts}       
\usepackage{amsmath}
\usepackage{amsthm}
\usepackage{bm}
\usepackage[ruled,vlined,linesnumbered]{algorithm2e}
\usepackage{marvosym}
\usepackage{multirow}
\usepackage{xcolor}
\usepackage{color, soul}
\usepackage{csquotes}
\usepackage{natbib}

\newtheorem{theorem}{Theorem}
\newtheorem{definition}{Definition}
\newtheorem{lemma}{Lemma}

\title{Prompt Valuation based on Shapley Values}

\author{%
  Hanxi Liu \\
  Zhejiang University\\
  \texttt{liuhx01@zju.edu.cn} \\
  \And
  Xiaokai Mao \\
  Zhejiang University \\
  \texttt{xiaokaimao@zju.edu.cn} \\
  \And
  Haocheng Xia \\
  Zhejiang University \\
  \texttt{xiahc@zju.edu.cn} \\
  \And
  Jinfei Liu \\
  Zhejiang University \\
  \texttt{jinfeiliu@zju.edu.cn} \\
  \And
  Jian Lou \\
  Zhejiang University \\
  \texttt{jian.lou@zju.edu.cn} \\
  \And
  Kui Ren \\
  Zhejiang University \\
  \texttt{kuiren@zju.edu.cn}
}

\begin{document}

\maketitle

\begin{abstract}
Pre-trained language models exploit natural language prompts to solve new tasks without any parameter updates. 
Multi-prompt learning utilizes multiple prompts to perform a task, thus introducing diversity and reducing biases.
However, even with carefully selected prompts, only a part of them have a positive effect while some prompts can impair performance, meaning fewer prompts may achieve higher performance. Existing methods only initially select a fixed number of prompts, which cannot further refine the selection result or even determine how many to refine.
Can we find a value measurement for prompts to distinguish the value of each prompt to enhance task performance?
In this paper, we utilize the Shapley value, which uniquely satisfies four fairness properties, to quantify the value of prompts equitably.
Furthermore, to reduce the \#P-hard complexity of accurately calculating Shapley values, we propose an efficient learning-based method to predict Shapley values for prompt ensembling towards real-time prompt valuation.
Through experiments on sentiment analysis, arithmetic reasoning, and commonsense reasoning with BERT and GPT-3.5-turbo, we validate the effectiveness of the Shapley value in prompt valuation, as well as the efficiency of using the learning-based Shapley value estimation method.
\end{abstract}

\section{Introduction}\label{section:introduction}
Pre-trained language models~\cite{gpt3,palm,bert,gpt4,lamda,llama,llama2} have demonstrated remarkable success across a range of natural language processing tasks including text classification~\cite{DBLP:conf/acl/GaoFC20}, question answering~\cite{DBLP:journals/corr/abs-2105-05541}, and arithmetic reasoning~\cite{DBLP:conf/acl/ImaniD023}. 
As an alternative paradigm to fine-tuning, prompt learning offers an off-the-shelf way to perform downstream tasks~\cite{DBLP:journals/ijcv/ZhouYLL22}, eliminating the need to train and store parameters separately for each task which is costly~\cite{finetune}. 

While naive single prompt learning is promising and has indeed achieved success,
it needs elaborate methods or manual efforts of domain experts to find the most suitable prompt for specific tasks. 
Automatic searching methods normally require gradient information~\cite{autoprompt,optiprompt} or training additional parameters~\cite{profixtuning}, which are 
impossible for black-box access settings and display significantly high levels of bias~\cite{DBLP:journals/corr/abs-2403-09963}. 
At the same time, language models are highly sensitive to the prompt due to the instability, which means manually crafted prompts may achieve sub-optimal performance, even leading to low performance~\cite{HowCanWeKnow,rationaleaugment}.

To overcome the shortcomings of a single prompt, a basic view is that the task performance can be significantly improved if sufficient diversity is introduced~\cite{rationaleaugment}. Using multiple prompts can aggregate across diverse results to reduce bias towards specific labels~\cite{optiprompt}, overcome the brittleness of performance to sub-optimal prompts, and further improve the effectiveness of prompting methods without training additional parameters. This pattern is called multi-prompt learning~\cite{promptsurvey}. 
In this paper, we focus on two widely used ways to extend single-prompt learning to multi-prompt learning. One is prompt ensembling which aggregates multiple outputs~\cite{voting1,voting2,sc}. It can be used for prompts with different formats, tasks, and language model families. The other is prompt augmentation, a generally acknowledged way to improve the reasoning capabilities of large language models, which provides a few additional examples that can take advantage of the ability of language models to learn repetitive patterns~\cite{gpt3,wei22cot}.

\paragraph{Motivation} Using multiple high-value prompts proves beneficial as it enables models to capture various information, enhancing their ability to handle diverse inputs~\cite{HowCanWeKnow}. 
In contrast, using multiple low-value prompts may lead to limited or even negative effects. Such prompts struggle to address the shortcomings of single prompts. 
The prompts generated or selected by~\cite{activecot,complexcot,wei22cot,autocot} may include both two types of prompts. By further refining these prompts, it is possible to use only a part of them to achieve higher task performance. \emph{Can we find a value measurement for prompts to distinguish the value of each prompt to enhance task performance?}

%


\paragraph{Contributions} In this paper, we adopt the Shapley value which uniquely satisfies four desired fairness properties, to quantify the value of prompts in multi-prompt learning equitably.
Our research demonstrates the effectiveness of this method, as it effectively distinguishes and evaluates the contribution of each prompt.
Using the Shapley value to value prompts does not require any fine-tuning~\cite{star}, verifiers~\cite{verify}, calibrators~\cite{calibrate}, or additional datasets~\cite{DBLP:journals/corr/abs-2206-02336}. It is unspecific to language models and prompt types. When calculating the Shapley value, the coalition and utility function can be freely defined, therefore it is suitable for prompt valuation.

Furthermore, to reduce the complexity of calculating the Shapley value, which is \#P-hard~\cite{DBLP:journals/mor/DengP94}, we propose a learning-based method for real-time prompt valuation for prompt ensembling with smaller masked language models that require more prompts to enhance performance while excelling at end-to-end tasks. Based on the viewpoint that the semantic scope of prompts for a specific task is limited, Shapley values of prompts for one task can be learned and predicted by calculating Shapley values on existing prompts in the specific task and training a regression predictor. We show that the Shapley value can be learned for certain tasks. 

The Shapley value~\cite{shapley1953value} measures each player's contribution to the cooperative game theory.
It has been widely used in machine learning model explanation~\cite{SHAP}, feature selection~\cite{featselect}, and data pricing~\cite{dealer}.
The Shapley value of one player is the weighted average of all its marginal contributions, which is the utility differences between player sets with and without the player. The Shapley value captures all possible cooperation scenarios of each player and hence can distinguish valuable players from worthless players.

The valuation results can provide us with insights into which prompts have a more significant impact on task performance and output quality, thereby guiding us in designing and optimizing prompt combinations more effectively. By understanding the relative contributions of each prompt, we can tailor prompt combinations more precisely, ensuring that only the most crucial prompts get emphasized.
At the same time, considering that prompts have become commodities in the data market and now there are some prompt markets~\cite{pmtai,pmtbase,pmtrr,pmtwink}, quantifying the value of prompts will guide more reasonable profit distribution. 

The experimental results demonstrate the effectiveness and efficiency of our methods. On the one hand, our valuation method can quantify the value of each prompt equitably so that the task performance is enhanced with fewer prompts. On the other hand, our learning-based method can greatly reduce the computation time of Shapley values, while maintaining accuracy. We briefly summarize our contributions as follows.

\begin{itemize}
    \item We incorporate the Shapley value in multi-prompt learning for equitable prompt valuation to identify valuable or worthless prompts and find the prompt combination with fewer but more performant prompts.
   
    \item We propose a learning-based method to predict Shapley values of prompts in real time for prompt ensembling and prove Shapley values of prompts for a specific task are learnable.

    \item Experiments on sentiment analysis, arithmetic reasoning, and commonsense reasoning with BERT and GPT-3.5-turbo demonstrate the effectiveness of our valuation method and the efficiency of our learning-based method.
\end{itemize}

\section{Related works}\label{section:relatedworks}
Prompts can be divided into discrete prompts and continuous prompts based on whether they are natural language. The latter often requires additional training and is not suitable for large language models, so we focus on the former in this paper. 
Discrete prompts play a role by concatenating with the original input to form a new sentence. For traditional discrete prompts, there is a slot for predicting the answer, and depending on the slot's position in the new sentence, it can be divided into cloze prompts and prefix prompts. The methods for searching the discrete prompts typically require a large text corpus~\cite{HowCanWeKnow}, gradient information~\cite{autoprompt}, or additional pre-trained language models~\cite{DBLP:journals/tacl/Ben-DavidOR22}. 
Few-shot prompting is another form of discrete prompting that only performs well on large language models. It prepends a few examples to the input so that the language model can learn and generalize repetitive patterns. The models can quickly adapt to novel tasks without fine-tuning or additional training by conditioning the model on a few examples. The examples of standard few-shot prompts typically consist of questions and answers~\cite{gpt3}. To further utilize the emergent abilities of large models to perform complex reasoning tasks, rationales are added into the examples~\cite{wei22cot}. This kind of prompt is called chain-of-thought (CoT).

The key of multi-prompt learning is to choose the right set of prompts. For discrete prompts, it is essential to design a series of templates, a process that used to be done manually. Now it can be done with automatic generation tools. For examples used for few-shot prompting, there are various research lines to improve the example selection. Such as similarity-based prompt selection which calculates the similarity between examples and retrieves the representative examples as the prompt~\cite{DBLP:conf/naacl/RubinHB22,DBLP:conf/iclr/SuKWSWX0OZS023,autocot} and probability-based prompt selection methods~\cite{probselect}.
Our method based on Shapley values is an extension of the existing methods, used to further refine their selection results.

Several benchmarks~\cite{chatbotarena,cothub,helm,bigbench,glue} can evaluate the prompts indirectly. Specifically, ~\cite{bigbench,glue} are basic large model benchmarks. ~\cite{cothub} focuses on evaluating reasoning while ~\cite{helm} evaluates a significantly wider spectrum of tasks. ~\cite{chatbotarena} evaluate the dialog user preference.  From another perspective, ~\cite{eval} proposes an evaluation suite that aggregates metrics across models to evaluate prompts over a diverse set of tasks and models, facilitating objective comparison. Our method can evaluate prompts on most models and tasks while obtaining a performant prompt combination. At the same time, as a widely used attribution method, The Shapley value can consider the interaction between prompts for equitable evaluation, improve interpretability, and glean additional insights.

\section{Shapley values for prompts}\label{sec:prop}

To equitably value the prompts, we introduce the Shapley value to quantify the value of prompts.
We first overview the relevant concept and define symbols in Section \ref{subsec:pre}. 
Then we briefly introduce how prompt and multiple prompts work in Section \ref{subsec:multi}.
Lastly, we propose the Shapley value in multi-prompt learning in Section \ref{subsec:PSV}. 

\subsection{Preliminaries}\label{subsec:pre}
Suppose we have a pre-trained language model $\mathcal{M}$.
Note that we primarily focus on masked language models represented by the BERT family and auto-regressive language models represented by the GPT family. These two types of models are typically employed for inference (understanding) and reasoning (generation) tasks, respectively. When referring to these tasks in the following text, we default to using the corresponding type of language model.

To guide the model towards desired outputs, we have a set of prompts $\mathcal{P}=\{p_1,\ldots,p_n\}$ which are crafted manually or generated using automated methods and in the form of templates or exemplars.
A natural way to quantify the contribution of prompts is evaluating their performance and obtaining a score (e.g., accuracy, $F_1$ score) on a validation set $\mathcal{V}=\{\bm{d}_1, \ldots, \bm{d}_l\}$ and for each instance $\bm{d}_k = (x_k, y_k)$ in $\mathcal{V}$, $x_k$ are sentences to be understood in inference tasks or questions in reasoning tasks while $y_k$ is the ground-truth answer. 
We denote the performance (utility function) as $\mathcal{U}(\cdot)$ and define it as follows.
\begin{equation}\label{equ:sv}
\mathcal{U}(\mathcal{S}) = \frac{1}{|\mathcal{V}|} \sum_{\bm{d}_k \in \mathcal{V}} \mathcal{I}(y_k, \mathcal{M}([x_k, \mathcal{S}])),
\end{equation}
where $\mathcal{S} (\mathcal{S} \subseteq \mathcal{P})$ is the prompt coalition
. $[x_k, \mathcal{S}]$ represents the inputs for querying $\mathcal{M}$, obtained by concatenating $x_k$ with all $p$ in $\mathcal{S}$. $\mathcal{I}(\cdot)$ is a discriminant function that returns a value based on the comparison between the output of $\mathcal{M}$ and the ground truth. 

\subsection{Multi-prompt learning}\label{subsec:multi}
We first briefly recall how prompts work in inference and reasoning tasks. For the former, we take the sentiment analysis task as an example. Consider an original input $x$ (e.g., ``I like this film.'') and we need to analyze its sentiment (positive or negative), we first construct an input $\hat{x}$ by concatenating it with a prompt $p$ (e.g., ``The sentiment is \texttt{[MASK]}.'') so that the model can know what it needs to do. Then we query the pre-trained language model to predict the word that could potentially appear at \texttt{[MASK]}. Lastly, we determine the sentiment conveyed by this word and assign it a positive or negative label. For the latter, given an example question $x$ (e.g., ``[5 + ? × 19 - 15 - 7]/[13 × 13 - 156] = 6''), we place similar examples that can demonstrate how to solve problems as the prompts in front of $x$ to elicit reasoning and then get a new input $\hat{x}$ so that we can query the model and obtain a more reliable answer.

Multiple prompts can be employed for predicting the same instance by more fully leveraging the inherent knowledge of language models to reduce bias, make the results more reliable, and enhance overall performance. In this paper, we specifically focus on voting in prompt ensembling.
and few-shot/few-shot-CoT in prompt augmentation.
For convenience, we will still refer to them as prompt ensembling and prompt augmentation in the following context. 
The examples of the two methods are shown in Figure \ref{fig:pensemble}. 


\begin{figure}[ht]
    \centering
    \includegraphics[width=1\linewidth]{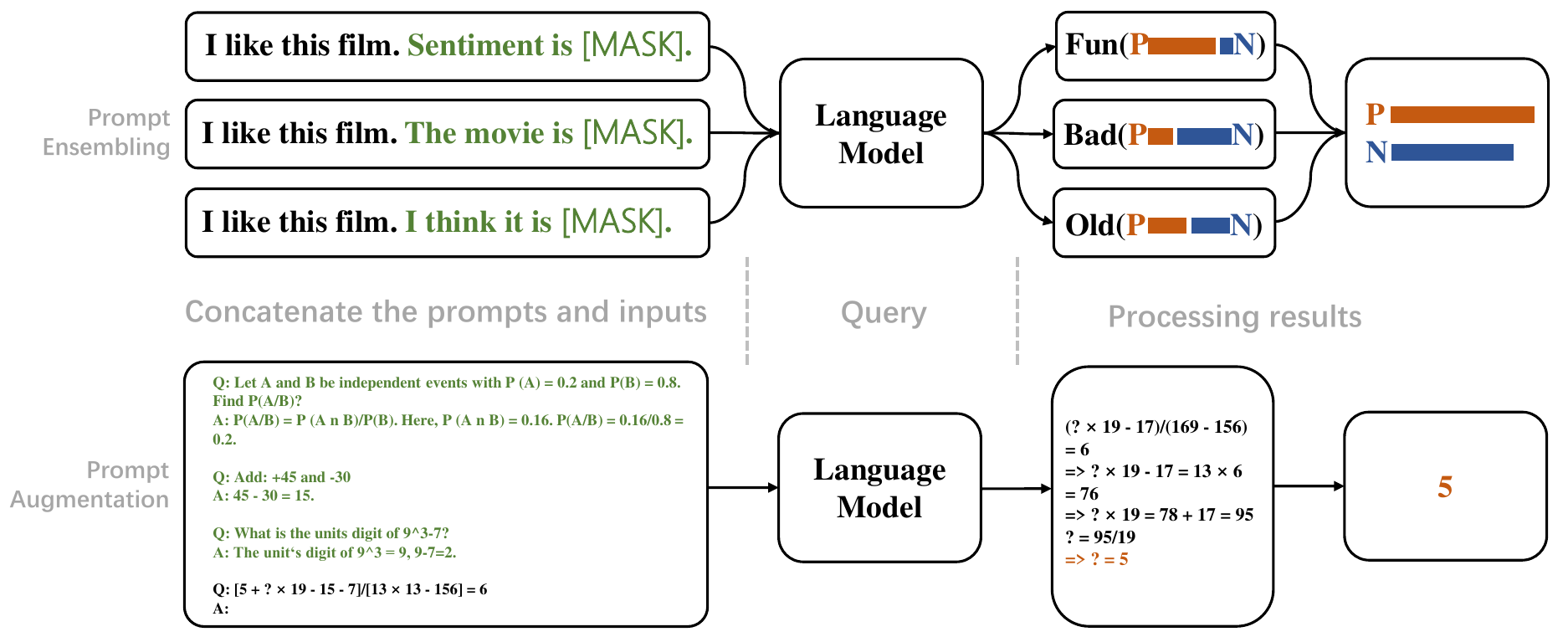}
    \caption{Examples of prompt ensembling and prompt augmentation.}
    \label{fig:pensemble}
\end{figure}

\subsection{Shapley values in multi-prompt learning}\label{subsec:PSV}

The Shapley value is introduced into machine learning by using prediction accuracy on a validation set as the utility function~\cite{datashapley,DBLP:conf/icde/0006XSL0P023,DBLP:journals/pacmmod/0006SL0P023}. In this way, each data point in the training set can be measured for its contribution or assessed for its value. Naturally, this paradigm can be applied to evaluate the value of prompts. 

\paragraph{The Shapley value} Recall that we have a set of $n$ prompts $\mathcal{P}=\{p_1,\ldots,p_n\}$. A \emph{coalition} $\mathcal{S}$ is a subset of $\mathcal{P}$ that cooperates to complete the task. Given a utility function $\mathcal{U}(\mathcal{S})$ evaluates the utility of a coalition $\mathcal{S}$ for the cooperative task, the marginal contribution of $p_i$ with respect to a coalition $\mathcal{S}$ $(p_i \notin \mathcal{S})$ is $\mathcal{U}(\mathcal{S}\cup \{p_i\})-\mathcal{U}(\mathcal{S})$. We have different definitions of the coalition between prompt ensembling and prompt augmentation. Refer to Figure \ref{fig:coalition} for the examples of coalition.

The Shapley value measures the expectation of marginal contribution by $p_i$ in all possible coalitions over $\mathcal{N}$. That is,
\begin{equation}\label{equ:SV}
  \mathcal{SV}_i=\frac{1}{n} \sum_{\mathcal{S}\subseteq \mathcal{N} \setminus \{p_i\}}   \frac{\mathcal{U}(\mathcal{S}\cup \{p_i\})-\mathcal{U}(\mathcal{S})}{\binom{n-1}{|\mathcal{S}|} }.
\end{equation} 
By computing Shapley values of prompts, we can fairly quantify their value in improving performance.

\begin{figure}[ht]
    \centering
    \includegraphics[width=0.8\linewidth]{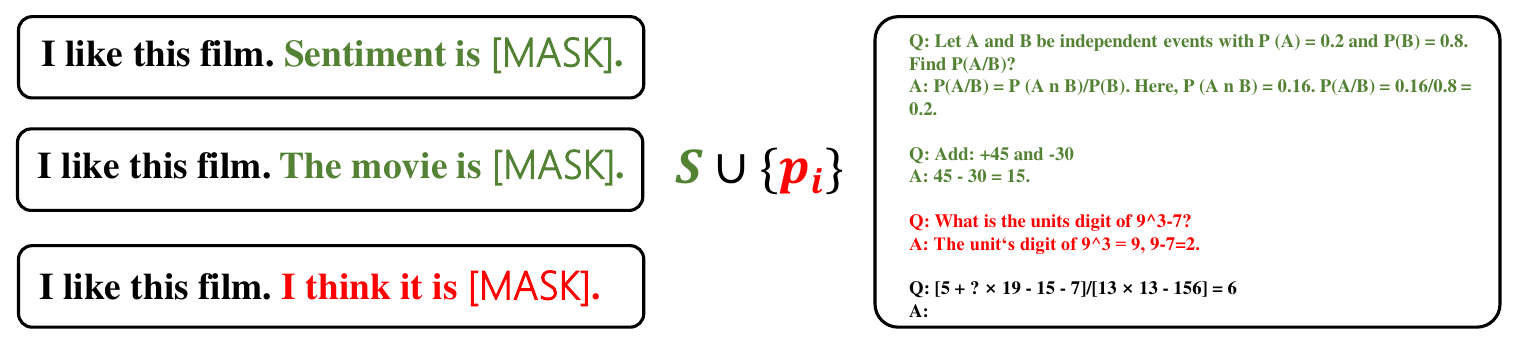}
    \caption{Examples of coalitions.}
    \label{fig:coalition}
\end{figure}

\section{Learning Shapley values of prompts}\label{section:learning}
To address the computational complexity of calculating Shapley values, especially for smaller models that can improve performance with a large number of prompts, we propose a machine learning-based method for prompt ensembling to predict Shapley values of prompts in two stages. In Section \ref{subsec:methodology}, we introduce the specific methodology, and we demonstrate the feasibility of the method in Section \ref{subsec:proof}.

\subsection{Methodology}\label{subsec:methodology}
For prompt sets related to the same task (e.g., sentiment analysis), we first calculate Shapley values of prompts in one set.
Then we use the predictions from the previous stage to train a simple model to predict Shapley values of prompts in another set.

Specifically, in the first stage, we use the Monte Carlo method~\cite{datashapley} to estimate Shapley values of prompts we already have. Considering texts cannot be used as inputs, we use a small pre-trained language model to encode the prompts into $d$-embeddings $\bm{\mathcal{E}} = [\bm{e}_1,\ldots,\bm{e}_n]\in\mathbb{R}^{|\mathcal{P}|\times d}$, where $\bm{e}_i$ is the embedding vector for prompt $p_i$, and $d$ is the dimensionality of the embedding.
Embeddings can effectively characterize and measure prompt similarity. For example, Sentence-BERT~\cite{DBLP:conf/emnlp/ReimersG19} measures the semantic similarity of prompts, while BERT~\cite{bert} measures that of tokens. Research indicates that in transformer-based models, each layer produces representations of different levels~\cite {DBLP:conf/emnlp/HaoDWX19}.

In the second stage, we use these embeddings as features and their calculated Shapley values as labels to obtain a training set and train a model $f$. When new prompts arrive, we encode them into embeddings and predict their Shapley values $\widehat{\mathcal{SV}_m} = f(\bm{e}_m;\bm{\theta})$, where $\widehat{\mathcal{SV}_m}$ is the predicted Shapley value for new prompt $m$ and $\bm{\theta}$ represents the parameters of the model. Refer to Algorithm \ref{alg:tmc} for details.

\begin{algorithm}
\caption{Learning Shapley values.}
\label{alg:tmc}\SetKwInOut{Input}{input}\SetKwInOut{Output}{output}
\SetCommentSty{itshape}
\Input{Prompt set $\mathcal{P} = \{p_1, \ldots, p_n\}$, \\ new prompt set $\mathcal{P}_{new} = \{pn_1, \ldots, pn_m\}$ \\ encoding model $\mathcal{M_E}$ \\ untrained model $f(\bm{e};\bm{\theta})$.}
\Output{Predicted Shapley values of new prompts $\widehat{\mathcal{SV}_1}, \dots, \widehat{\mathcal{SV}_m}$.} 
    \tcp{Prepare the training set}
    $\{\mathcal{SV}_1, \ldots, \mathcal{SV}_n\}$ $\gets$ Calculate through Monte Carlo method\;
    \For{$j$ = 1 to $n$}{
        $\bm{e}_j = \mathcal{M_E}(p_j)$\;
    }
    \tcp{Train the model to predict}
    Train $f$ with the training set $\{(\bm{e}_1, \mathcal{SV}_1), \ldots, (\bm{e}_n, \mathcal{SV}_n)\}$\;
    \tcp{Predict Shaple values}
    \For{$j$ = 1 to $m$}{
        $\bm{e}_j = \mathcal{M_E}(pn_j)$\;
        $\widehat{\mathcal{SV}_j} = f(\bm{e}_j;\bm{\theta})$\;
    }
    \Return $\widehat{\mathcal{SV}_1}, \dots, \widehat{\mathcal{SV}_m}$\;
\end{algorithm}

\subsection{Lipschitz continuity}\label{subsec:proof}
\paragraph{Similar prompts receive similar values}
Intuitively, we hope that if two prompts $p$ and $p^{\prime}$ are similar according to some appropriate metric, they will receive similar Shapley values. We confirm this intuition when the utility function is Lipschitz continuity.

\begin{definition}[\textbf{Lipschitz Continuity}]
A function $f:\mathbb{R}^n \rightarrow \mathbb{R}^m$ is said to be Lipschitz continuous if there exists a constant $L \geq 0$ such that for all $x, y \in \mathbb{R}^n$, $\|f(x) - f(y)\| \leq L\|x - y\|$.
\end{definition}
This definition implies that the change in the output of the function is bounded by a constant multiple of the change in the input. 

\paragraph{Lipschitz Continuity of the utility function }
As shown in Figure \ref{fig:pensemble}, the input of the utility function consists of a set of prompts, and each prompt yields a predicted probability for every $x_k$ in $\mathcal{V}$. Owing to the ensemble of prompts, an ensemble probability is assigned to each $x_k$. Subsequently, the final output is the overall accuracy of the dataset.

We simplify each prompt and its prediction into a sub-classifier $h_i$, for one data point $x$, each classifier gives a probabilistic prediction $h_i(x)$, and the final prediction $f(x)$ is given by the average of all classifier probabilities: $f(x) = \frac{1}{N}\sum_{i=1}^{N}h_i(x)$. Now, suppose that we modify the probabilistic prediction of the $k$-th classifier
from $h_k(x)$ to $h_k'(x)$, this will affect the output of the ensemble, becoming:
\begin{equation}
    f'(x)=\frac{1}{N}\left(\sum_{i=1}^{k-1}h_i(x)+h_k'(x)+\sum_{i=k+1}^{N}h_i(x)\right).
\end{equation}
Then we can calculate the difference between the prediction results of $f(x)$ and $f'(x)$
\begin{align}\label{eq4}
    |f'(x)-f(x)|&=\left|\frac{1}{N}\left(\sum_{i=1}^{k-1}h_i(x)+h_k'(x)+\sum_{i=k+1}^Nh_i(x)\right)-\frac{1}{N}\sum_{i=1}^Nh_i(x)\right|\\
    &=\frac{1}{N}|h_k'(x)-h_k(x)|.
\end{align}
This means that as $N$ increases, the impact of any single classifier change on the output of the overall ensemble will decrease accordingly. 
The Lipschitz constant of the ensemble can be thought of as $\frac{1}{N}$, which decreases as $N$ increases. Next, we discuss the bound.

\begin{definition}[\textbf{Beta Distribution}]
The probability density function(PDF) of the Beta distribution, for $0 \leq x \leq 1$ and shape parameters $\alpha, \beta \geq 0$, is a power function of the variable $x$ and of its reflection $(1-x)$ as follows:
\begin{equation*}
    f(p; \alpha, \beta) = \frac{p^{\alpha - 1} (1 - p)^{\beta - 1}}{B(\alpha, \beta)},
\end{equation*}
where the Beta function $B(\alpha, \beta) = \int_{0}^{1} t^{\alpha-1} (1-t)^{\beta-1} \, dt$.
\end{definition}

Then consider the entire dataset $\mathcal{V}$. Note that the binary prediction distribution \( P_{p_i} \) has the same support as a Beta distribution, allowing \( P_{p_i} \) to be mathematically modeled using a Beta distribution. 
Equation \ref{eq4} illustrates that changing a single classifier will not affect the predicted probability of one data point by more than $\frac{1}{N}$, recorded as $\epsilon$. Consider the worst case: after modifying a classifier, data points with a predicted result of $(0.5 \pm \epsilon, 0.5 \mp \epsilon)$ change all predictions from correct to incorrect. The cumulative probability of this interval is $P(0.5-\epsilon\leq X\leq0.5+\epsilon)=\int_{0.5-\epsilon}^{0.5+\epsilon}f(x;\alpha,\beta)dx$.
According to the central limit theorem, when the dataset $\mathcal{V}$ is large (i.e., the sample size for prediction is large), the Beta distribution can be approximated by a normal distribution $ X \sim N\left(\frac{\alpha}{\alpha+\beta},\sqrt{\frac{\alpha\beta}{(\alpha+\beta)^2(\alpha+\beta+1)}}\right) $.We can approximately calculate the cumulative probability in the interval $(0.5 - \epsilon, 0.5 + \epsilon)$:
\begin{align}\label{e6}
    P(0.5-\epsilon \leq X \leq0.5+\epsilon)&\approx\Phi\left(\frac{0.5+\epsilon-\mu}\sigma\right)-\Phi\left(\frac{0.5-\epsilon-\mu}\sigma\right)\\
    &\approx \frac{2\epsilon}{\sqrt{2\pi}\sigma}\left(1-\frac{(0.5-\mu)^2}{3\sigma^2}\right).
\end{align}
Where $\mu = \frac{\alpha}{\alpha+\beta}$, $\sigma = \sqrt{\frac{\alpha\beta}{(\alpha+\beta)^2(\alpha+\beta+1)}}$ and $\epsilon = \frac1N|h'_k(x)-h_k(x)|$.This indicates a probability $P$ that the model prediction value for a given data point falls within the interval $(0.5 - \epsilon, 0.5 + \epsilon)$. This represents the worst case where the model diminishes the accuracy of the dataset. We demonstrate in Appendix \ref{appendix:proofs} that this probability is still small and is related to $ \frac{1}{N} | h_k'(x) - h_k(x) | $. Additionally, the following experimental results show that randomly adding prompts to the prompt ensemble has almost no impact on the final accuracy.

\begin{theorem}\label{t1}
Let~$\mathcal{U}$ be a utility function. if ~$\mathcal{U}$ is Lipschitz continuous with respect to some norm $||\cdot||$ on the input space with a Lipschitz constant $L$, then for any two inputs $\bm{e}_1$ and $\bm{e}_2$ corresponding to similar prompts, the absolute difference in their Shapley value $\mathcal{SV}_i$ and $\mathcal{SV}_j$ is bounded by $L$ time the norm of the difference of the inputs, that is:
\begin{equation}\label{1}
|\mathcal{U}(S \cup \{\bm{e}_i\}) - \mathcal{U}(S \cup \{\bm{e}_j\})| \leq L \cdot ||\bm{e}_i - \bm{e}_j|| \implies  |\mathcal{SV}_i - \mathcal{SV}_j| \leq L \cdot ||\bm{e}_i - \bm{e}_j||,
\end{equation}
where $S$ is coalition of embeddings except $\bm{e}_i$ and $\bm{e}_j$. 
\end{theorem}
The complete proof of Theorem \ref{t1} can be found in the Appendix \ref{appendix:proofs}.
The Theorem quantitatively constrains the divergence in the Shapley values for two inputs by linking it directly to the inputs' disparity, scaled by the Lipschitz constant $L$. So Theorem \ref{t1} provides a theoretical justification for the learning method, intuitively, it guarantees that similar inputs(i.e., small $||\bm{e}_i - \bm{e}_j||$) imply similar outputs(i.e., $|\mathcal{SV}_i - \mathcal{SV}_j|$). Therefore, Shapley values of prompts in prompt ensembling are learnable.

\section{Experiments}\label{section:experiments}
In this section, we conduct a series of experiments and present the empirical evaluation of the proposed methods. In Section \ref{subsec:setup},  we provide details of the experimental setup.
In Sections \ref{resadd}, we conduct experiments to verify the effectiveness of the Shapley value. 
In Section \ref{reslearn}, we demonstrated the efficiency of our learning-based method.


\subsection{Experimental setup}\label{subsec:setup}

\paragraph{Datasets and tasks}
For prompt ensembling, we conducted experiments on the following natural language inference task: Stanford Sentiment Treebank (SST2, sentiment analysis)~\cite{sst2}, which is also used in our learning-based method. For prompt augmentation, we focus on complex reasoning tasks: AQuA (arithmetic reasoning)~\cite{aqua} and Bigbench Date (commonsense reasoning)~\cite{bigbench}.
The evaluation metric for all tasks is the accuracy.

\paragraph{Models}
For prompt ensembling, our experiments are based on BERT-base (110M parameters)~\cite{bert} which has had a significant impact on the field of natural language processing in its early stages. We use the checkpoint provided by HuggingFace~\cite{huggingface}. For prompt augmentation, we utilize GPT-3.5-turbo, which is suitable at the time we conduct the experiments. It strikes a good trade-off between convenience, capability, and price. We directly call the OpenAI API to use the model.


\paragraph{Prompts}
All prompts used for the inference tasks and the subsequent learning-based method are generated by ChatGPT. There is no need to consider the content of the prompts, we only want to value them. As for the reasoning tasks, the prompts come from~\cite{wei22cot} which are reused in the following work. In this way, we can conveniently compare results with related work.


\paragraph{Implements}
For inference tasks, our specific implementation mainly follows ~\cite{autoprompt} including result generation and mapping. We use ``\{sentence\} \{prompt\} \texttt{[MASK]}.'' as the template. For reasoning tasks, we set \text{temperature} to 0 with greedy decode and the input formats are the same as ~\cite{wei22cot}.

\subsection{Results for prompt valuation}\label{resadd}
\paragraph{Identifying valuable prompts}
To demonstrate the effectiveness of the Shapley value in evaluating the value of prompts, we sort prompts from high to low based on their Shapley values, adding one prompt to perform the task each time. We also conducted leave-one-out experiments for comparison. Specific results can be found in Figure \ref{fig:addaug}. When prompts with higher Shapley values are added, they have a positive impact on the final prediction. However, as prompts with lower Shapley values are added, the information provided by the prompts may become inaccurate, irrelevant, or misleading, which interferes with the model's predictions, leading to a decline in the quality of the prediction results. 
The results suggest that simply increasing the number of prompts does not lead to better results. The key is to choose prompt combinations with higher values so that we can achieve higher performance with fewer valuable prompts.

\paragraph{Competitive performance of selected prompt combination}
To intuitively demonstrate the superiority of our method, we compare the highest accuracy achieved when using the selected prompt combination with baselines, the results are shown in Table \ref{tab:acc}. The suitable prompt combination selected through the Shapley value can not only perform tasks effectively with fewer prompts but also remain competitive performance. For AQuA (Date), the performance of 4 (6) shots is lower than that of 2 (3) shots. it avoids a large number of prompts occupying much of the input space for prompt augmentation which is worth considering for potentially complex tasks.

More intuitively, for prompt ensembling used in SST2, we only need to vote with less than 200 high-value prompts to achieve performance close to that of well-designed automated methods. More importantly, these prompts can be fully generated by chatbots, which is almost zero cost compared to automated methods that take a lot of time to search for an optimal prompt.




\begin{figure}[htbp]
    
    \centering
            \centering
            \includegraphics[width=\textwidth]{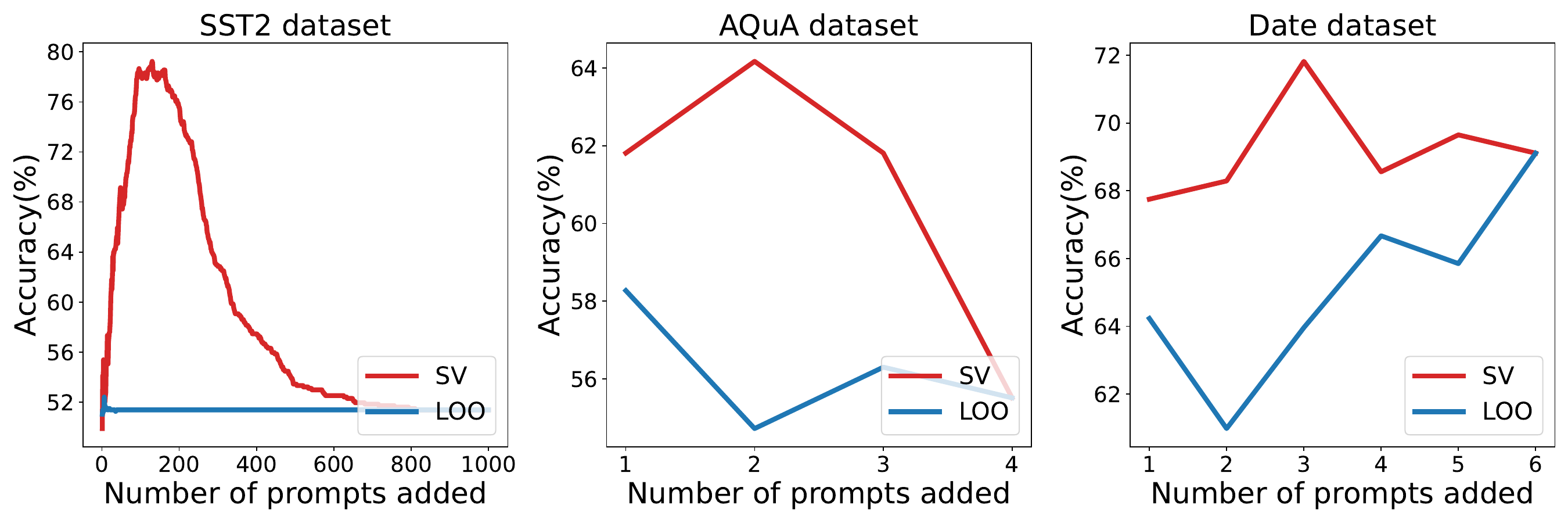}
        \caption{Results of SST2 with BERT-base, as well as AQuA and Date with GPT-3.5-turbo and Manual-CoT. We add the currently most valuable prompt to the combination iteratively. For comparison, we also calculate the leave-one-out (LOO) value and combine prompts in the same manner.}
        \label{fig:addaug}
\end{figure}

\begin{table}[htbp]
  \caption{We compare our methods with baselines: AutoPrompt~\cite{autoprompt} and Manual (reported in~\cite{autoprompt}) for SST2, Manual-CoT~\cite{wei22cot} for AQuA and Date, Manual (prompts from~\cite{wei22cot} but without rationale), Auto-CoT~\cite{autocot}, Active-CoT~\cite{activecot} and Complex-CoT (9 steps)~\cite{complexcot} for AQuA.}
  \label{tab:acc}
  \centering
  \begin{tabular}{lccc}
    \toprule
    ~ & SST2  & AQuA & Date \\
    \midrule
    AutoPrompt/Manual  & 80.9/63.2 & ~ & ~ \\
    Ours  & 78.9 & ~ & ~ \\
    \midrule
    Manual-CoT (+Ours) & ~ & 55.5 (64.2) & 69.1 (71.8) \\
    Manual (+Ours) & ~ & 30.3 (35.4) & 53.7 (57.5) \\
    Auto-CoT (+Ours)  & ~ & 59.1 (60.6) & - \\
    Active-CoT (+Ours) & ~ & 56.3 (59.4) & - \\
    Complex-CoT (+Ours) & ~ & 61.8 (63.8) & - \\
    \bottomrule
  \end{tabular}
\end{table}



\subsection{Results for learning Shapley values}\label{reslearn}
We conducted empirical experiments to verify the effectiveness of our learning-based methods. 
We use the embeddings of prompts that are used for the SST2 as training data, and their Shapley values as labels.
Specifically, we follow~\cite{DBLP:journals/corr/abs-1906-04165} and average the outputs from BERT's last four layers as the embedding to obtain a richer semantic representation of the prompt. In contrast, Sentence-BERT (SBERT) derives semantically meaningful sentence embeddings from the last layer, and we use them without further processing.
Then, we train the Bayesian Ridge regressor, Gaussian Process regressor, and Linear regressor to fit the data, respectively.
Lastly, we use the additional 100 prompts as the test set, encoding them into embeddings in the same way as the training prompts, and taking their Shapley values estimated using a huge number of Monte Carlo iterations as the truth Shapley values for comparison.

We first calculate the Pearson coefficient between predicted Shapley values and truth Shapley values. As shown in Table \ref{tab:rank}, Shapley values predicted using SBERT embeddings are closer to truth Shapley values overall, and they are more stable for the machine learning model compared to Shapley values predicted using BERT embeddings, which means the well-designed semantic representation is more representative.
We also conduct sorting and adding experiments as mentioned above. The results can be found in Figure \ref{fig:pred}. The accuracy curve obtained through SBERT embeddings is more similar to the shape of the truth accuracy curve. It's worth noting that the Pearson coefficient of Shapley values predicted by BERT embeddings is lower, but with only about 20 prompts, the accuracy is greater than 80\%, even surpassing AutoPrompt. We assume that the method of sorting and adding prompts one by one cannot obtain the global optimal prompt combination, related research can be seen in~\cite{wang2024rethinking}, which is a possible direction for future research.

\begin{table}
  \caption{The Pearson coefficient between Shapley values estimated by the Monte Carlo method and Shapley values predicted by different machine learning models and embeddings.}
  \label{tab:rank}
  \centering
  \begin{tabular}{lccc}
    \toprule
    ~ & Bayesian Ridge & Gaussian Process & Linear Regression \\
    \midrule
    Embeddings from BERT & 0.72 & 0.57 & 0.66 \\ 
    Embeddings from SBERT & 0.71 & 0.74 & 0.70 \\
    \bottomrule
  \end{tabular}
\end{table}

\begin{figure}[ht]
    \centering
    \includegraphics[width=1\linewidth]{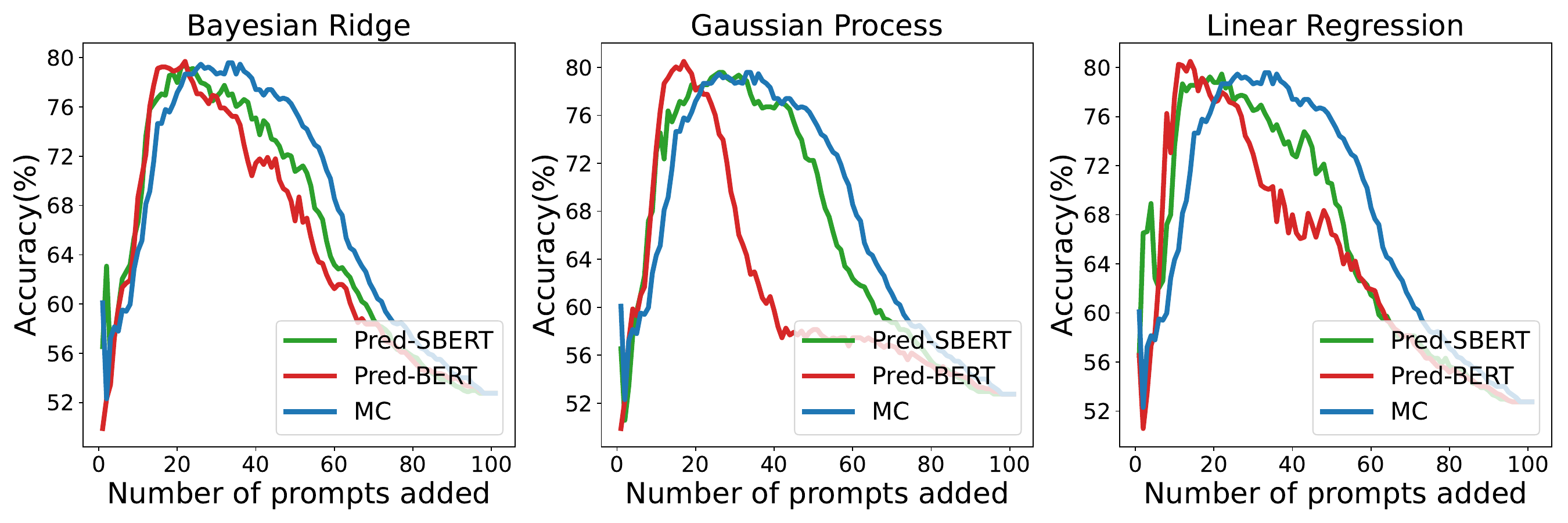}
    \caption{Sort prompt based on Shapley values obtained by the three methods, add prompts, and calculate accuracy separately on SST2.}
    \label{fig:pred}
\end{figure}

\section{Conclusion, limitations and future work}\label{section:conclusion}
In this paper, we apply the Shapley value to quantify the value of prompts, thus we can evaluate each prompt, obtaining a combination with fewer but performant prompts.
Our method can be an extension of the existing selection methods and can be used for any selected prompts to choose fewer prompts that achieve higher performance.
We also propose a learning-based method to alleviate the burden of calculating the Shapley value when the number of prompts is large and dynamic. Experiments show the effectiveness and efficiency. Overall, we believe that using Shapley values for prompt valuation is worth further research. 

Inevitably, implementing our method needs the complicated calculation of Shapley values. Although we have proposed a learning-based method for real-time prompt valuation, it does not apply to large language models, because there are usually not enough prompts to avoid overfitting or improve generalization and resource consumption is also unacceptable. Therefore, it is necessary to tailor methods specific to large language models to simplify computations. In addition, future research might discuss how to more broadly apply our methods in multi-prompt learning, such as different types of inputs in multimodal models. This involves how to define coalition and how to define appropriate utility functions. 

\bibliographystyle{apalike}
\bibliography{prompt}

\begin{thebibliography}{}

\bibitem[Ajith et~al., 2023]{eval}
Ajith, A., Pan, C., Xia, M., Deshpande, A., and Narasimhan, K. (2023).
\newblock Instructeval: Systematic evaluation of instruction selection methods.
\newblock {\em CoRR}, abs/2307.00259.

\bibitem[Ben{-}David et~al., 2022]{DBLP:journals/tacl/Ben-DavidOR22}
Ben{-}David, E., Oved, N., and Reichart, R. (2022).
\newblock {PADA:} example-based prompt learning for on-the-fly adaptation to unseen domains.
\newblock {\em Trans. Assoc. Comput. Linguistics}, 10:414--433.

\bibitem[Brown et~al., 2020]{gpt3}
Brown, T.~B., Mann, B., Ryder, N., Subbiah, M., Kaplan, J., Dhariwal, P., Neelakantan, A., Shyam, P., Sastry, G., Askell, A., Agarwal, S., Herbert{-}Voss, A., Krueger, G., Henighan, T., Child, R., Ramesh, A., Ziegler, D.~M., Wu, J., Winter, C., Hesse, C., Chen, M., Sigler, E., Litwin, M., Gray, S., Chess, B., Clark, J., Berner, C., McCandlish, S., Radford, A., Sutskever, I., and Amodei, D. (2020).
\newblock Language models are few-shot learners.
\newblock In Larochelle, H., Ranzato, M., Hadsell, R., Balcan, M., and Lin, H., editors, {\em Advances in Neural Information Processing Systems 33: Annual Conference on Neural Information Processing Systems 2020, NeurIPS 2020, December 6-12, 2020, virtual}.

\bibitem[Chiang et~al., 2024]{chatbotarena}
Chiang, W., Zheng, L., Sheng, Y., Angelopoulos, A.~N., Li, T., Li, D., Zhang, H., Zhu, B., Jordan, M.~I., Gonzalez, J.~E., and Stoica, I. (2024).
\newblock Chatbot arena: An open platform for evaluating llms by human preference.
\newblock {\em CoRR}, abs/2403.04132.

\bibitem[Chowdhery et~al., 2023]{palm}
Chowdhery, A., Narang, S., Devlin, J., Bosma, M., Mishra, G., Roberts, A., Barham, P., Chung, H.~W., Sutton, C., Gehrmann, S., et~al. (2023).
\newblock Palm: Scaling language modeling with pathways.
\newblock {\em Journal of Machine Learning Research}, 24(240):1--113.

\bibitem[Deng and Papadimitriou, 1994]{DBLP:journals/mor/DengP94}
Deng, X. and Papadimitriou, C.~H. (1994).
\newblock On the complexity of cooperative solution concepts.
\newblock {\em Math. Oper. Res.}, 19(2):257--266.

\bibitem[Diao et~al., 2023]{activecot}
Diao, S., Wang, P., Lin, Y., and Zhang, T. (2023).
\newblock Active prompting with chain-of-thought for large language models.
\newblock {\em CoRR}, abs/2302.12246.

\bibitem[Fryer et~al., 2021]{featselect}
Fryer, D.~V., Str{\"{u}}mke, I., and Nguyen, H.~D. (2021).
\newblock Shapley values for feature selection: The good, the bad, and the axioms.
\newblock {\em {IEEE} Access}, 9:144352--144360.

\bibitem[Fu et~al., 2023a]{cothub}
Fu, Y., Ou, L., Chen, M., Wan, Y., Peng, H., and Khot, T. (2023a).
\newblock Chain-of-thought hub: {A} continuous effort to measure large language models' reasoning performance.
\newblock {\em CoRR}, abs/2305.17306.

\bibitem[Fu et~al., 2023b]{complexcot}
Fu, Y., Peng, H., Sabharwal, A., Clark, P., and Khot, T. (2023b).
\newblock Complexity-based prompting for multi-step reasoning.
\newblock In {\em The Eleventh International Conference on Learning Representations, {ICLR} 2023, Kigali, Rwanda, May 1-5, 2023}. OpenReview.net.

\bibitem[Gao et~al., 2021]{DBLP:conf/acl/GaoFC20}
Gao, T., Fisch, A., and Chen, D. (2021).
\newblock Making pre-trained language models better few-shot learners.
\newblock In Zong, C., Xia, F., Li, W., and Navigli, R., editors, {\em Proceedings of the 59th Annual Meeting of the Association for Computational Linguistics and the 11th International Joint Conference on Natural Language Processing, {ACL/IJCNLP} 2021, (Volume 1: Long Papers), Virtual Event, August 1-6, 2021}, pages 3816--3830. Association for Computational Linguistics.

\bibitem[Ghorbani and Zou, 2019]{datashapley}
Ghorbani, A. and Zou, J.~Y. (2019).
\newblock Data shapley: Equitable valuation of data for machine learning.
\newblock In Chaudhuri, K. and Salakhutdinov, R., editors, {\em Proceedings of the 36th International Conference on Machine Learning, {ICML} 2019, 9-15 June 2019, Long Beach, California, {USA}}, volume~97 of {\em Proceedings of Machine Learning Research}, pages 2242--2251. {PMLR}.

\bibitem[Hambardzumyan et~al., 2021]{voting1}
Hambardzumyan, K., Khachatrian, H., and May, J. (2021).
\newblock {WARP:} word-level adversarial reprogramming.
\newblock In Zong, C., Xia, F., Li, W., and Navigli, R., editors, {\em Proceedings of the 59th Annual Meeting of the Association for Computational Linguistics and the 11th International Joint Conference on Natural Language Processing, {ACL/IJCNLP} 2021, (Volume 1: Long Papers), Virtual Event, August 1-6, 2021}, pages 4921--4933. Association for Computational Linguistics.

\bibitem[Hao et~al., 2019]{DBLP:conf/emnlp/HaoDWX19}
Hao, Y., Dong, L., Wei, F., and Xu, K. (2019).
\newblock Visualizing and understanding the effectiveness of {BERT}.
\newblock In Inui, K., Jiang, J., Ng, V., and Wan, X., editors, {\em Proceedings of the 2019 Conference on Empirical Methods in Natural Language Processing and the 9th International Joint Conference on Natural Language Processing, {EMNLP-IJCNLP} 2019, Hong Kong, China, November 3-7, 2019}, pages 4141--4150. Association for Computational Linguistics.

\bibitem[Imani et~al., 2023]{DBLP:conf/acl/ImaniD023}
Imani, S., Du, L., and Shrivastava, H. (2023).
\newblock Mathprompter: Mathematical reasoning using large language models.
\newblock In Sitaram, S., Klebanov, B.~B., and Williams, J.~D., editors, {\em Proceedings of the The 61st Annual Meeting of the Association for Computational Linguistics: Industry Track, {ACL} 2023, Toronto, Canada, July 9-14, 2023}, pages 37--42. Association for Computational Linguistics.

\bibitem[Jiang et~al., 2020]{HowCanWeKnow}
Jiang, Z., Xu, F.~F., Araki, J., and Neubig, G. (2020).
\newblock How can we know what language models know.
\newblock {\em Trans. Assoc. Comput. Linguistics}, 8:423--438.

\bibitem[Kenton and Toutanova, 2019]{bert}
Kenton, J. D. M.-W.~C. and Toutanova, L.~K. (2019).
\newblock Bert: Pre-training of deep bidirectional transformers for language understanding.
\newblock In {\em Proceedings of NAACL-HLT}, pages 4171--4186.

\bibitem[Lester et~al., 2021]{voting2}
Lester, B., Al{-}Rfou, R., and Constant, N. (2021).
\newblock The power of scale for parameter-efficient prompt tuning.
\newblock In Moens, M., Huang, X., Specia, L., and Yih, S.~W., editors, {\em Proceedings of the 2021 Conference on Empirical Methods in Natural Language Processing, {EMNLP} 2021, Virtual Event / Punta Cana, Dominican Republic, 7-11 November, 2021}, pages 3045--3059. Association for Computational Linguistics.

\bibitem[Li and Liang, 2021]{profixtuning}
Li, X.~L. and Liang, P. (2021).
\newblock Prefix-tuning: Optimizing continuous prompts for generation.
\newblock In Zong, C., Xia, F., Li, W., and Navigli, R., editors, {\em Proceedings of the 59th Annual Meeting of the Association for Computational Linguistics and the 11th International Joint Conference on Natural Language Processing, {ACL/IJCNLP} 2021, (Volume 1: Long Papers), Virtual Event, August 1-6, 2021}, pages 4582--4597. Association for Computational Linguistics.

\bibitem[Li et~al., 2022]{DBLP:journals/corr/abs-2206-02336}
Li, Y., Lin, Z., Zhang, S., Fu, Q., Chen, B., Lou, J., and Chen, W. (2022).
\newblock On the advance of making language models better reasoners.
\newblock {\em CoRR}, abs/2206.02336.

\bibitem[Liang et~al., 2022]{helm}
Liang, P., Bommasani, R., Lee, T., Tsipras, D., Soylu, D., Yasunaga, M., Zhang, Y., Narayanan, D., Wu, Y., Kumar, A., Newman, B., Yuan, B., Yan, B., Zhang, C., Cosgrove, C., Manning, C.~D., R{\'{e}}, C., Acosta{-}Navas, D., Hudson, D.~A., Zelikman, E., Durmus, E., Ladhak, F., Rong, F., Ren, H., Yao, H., Wang, J., Santhanam, K., Orr, L.~J., Zheng, L., Y{\"{u}}ksekg{\"{o}}n{\"{u}}l, M., Suzgun, M., Kim, N., Guha, N., Chatterji, N.~S., Khattab, O., Henderson, P., Huang, Q., Chi, R., Xie, S.~M., Santurkar, S., Ganguli, S., Hashimoto, T., Icard, T., Zhang, T., Chaudhary, V., Wang, W., Li, X., Mai, Y., Zhang, Y., and Koreeda, Y. (2022).
\newblock Holistic evaluation of language models.
\newblock {\em CoRR}, abs/2211.09110.

\bibitem[Lightman et~al., 2023]{verify}
Lightman, H., Kosaraju, V., Burda, Y., Edwards, H., Baker, B., Lee, T., Leike, J., Schulman, J., Sutskever, I., and Cobbe, K. (2023).
\newblock Let's verify step by step.
\newblock {\em CoRR}, abs/2305.20050.

\bibitem[Ling et~al., 2017]{aqua}
Ling, W., Yogatama, D., Dyer, C., and Blunsom, P. (2017).
\newblock Program induction by rationale generation: Learning to solve and explain algebraic word problems.
\newblock In Barzilay, R. and Kan, M., editors, {\em Proceedings of the 55th Annual Meeting of the Association for Computational Linguistics, {ACL} 2017, Vancouver, Canada, July 30 - August 4, Volume 1: Long Papers}, pages 158--167. Association for Computational Linguistics.

\bibitem[Liu et~al., 2021]{dealer}
Liu, J., Lou, J., Liu, J., Xiong, L., Pei, J., and Sun, J. (2021).
\newblock Dealer: An end-to-end model marketplace with differential privacy.
\newblock {\em Proc. {VLDB} Endow.}, 14(6):957--969.

\bibitem[Liu et~al., 2023]{promptsurvey}
Liu, P., Yuan, W., Fu, J., Jiang, Z., Hayashi, H., and Neubig, G. (2023).
\newblock Pre-train, prompt, and predict: {A} systematic survey of prompting methods in natural language processing.
\newblock {\em {ACM} Comput. Surv.}, 55(9):195:1--195:35.

\bibitem[Lundberg and Lee, 2017]{SHAP}
Lundberg, S.~M. and Lee, S. (2017).
\newblock A unified approach to interpreting model predictions.
\newblock In Guyon, I., von Luxburg, U., Bengio, S., Wallach, H.~M., Fergus, R., Vishwanathan, S. V.~N., and Garnett, R., editors, {\em Advances in Neural Information Processing Systems 30: Annual Conference on Neural Information Processing Systems 2017, December 4-9, 2017, Long Beach, CA, {USA}}, pages 4765--4774.

\bibitem[Miller, 2019]{DBLP:journals/corr/abs-1906-04165}
Miller, D. (2019).
\newblock Leveraging {BERT} for extractive text summarization on lectures.
\newblock {\em CoRR}, abs/1906.04165.

\bibitem[OpenAI, 2023]{gpt4}
OpenAI (2023).
\newblock {GPT-4} technical report.
\newblock {\em CoRR}, abs/2303.08774.

\bibitem[PromptAI, 2024]{pmtai}
PromptAI (2024).
\newblock {PromptAI}. \url{https://prompti.ai/}.
\newblock Accessed: 2024-05-22.

\bibitem[PromptBase, 2024]{pmtbase}
PromptBase (2024).
\newblock {PromptBase}. \url{https://promptbase.com/}.
\newblock Accessed: 2024-05-22.

\bibitem[Promptrr.io, 2024]{pmtrr}
Promptrr.io (2024).
\newblock {Promptrr.io}. \url{https://promptrr.io/}.
\newblock Accessed: 2024-05-22.

\bibitem[PromptWink, 2024]{pmtwink}
PromptWink (2024).
\newblock {PromptWink}. \url{https://www.promptwink.com/}.
\newblock Accessed: 2024-05-22.

\bibitem[Reimers and Gurevych, 2019]{DBLP:conf/emnlp/ReimersG19}
Reimers, N. and Gurevych, I. (2019).
\newblock Sentence-bert: Sentence embeddings using siamese bert-networks.
\newblock In Inui, K., Jiang, J., Ng, V., and Wan, X., editors, {\em Proceedings of the 2019 Conference on Empirical Methods in Natural Language Processing and the 9th International Joint Conference on Natural Language Processing, {EMNLP-IJCNLP} 2019, Hong Kong, China, November 3-7, 2019}, pages 3980--3990. Association for Computational Linguistics.

\bibitem[Rubin et~al., 2022]{DBLP:conf/naacl/RubinHB22}
Rubin, O., Herzig, J., and Berant, J. (2022).
\newblock Learning to retrieve prompts for in-context learning.
\newblock In Carpuat, M., de~Marneffe, M., and Ru{\'{\i}}z, I. V.~M., editors, {\em Proceedings of the 2022 Conference of the North American Chapter of the Association for Computational Linguistics: Human Language Technologies, {NAACL} 2022, Seattle, WA, United States, July 10-15, 2022}, pages 2655--2671. Association for Computational Linguistics.

\bibitem[Shapley, 1953]{shapley1953value}
Shapley, L.~S. (1953).
\newblock A value for n-person games.
\newblock {\em Contributions to the Theory of Games}, 2(28):307--317.

\bibitem[Sharma et~al., 2021]{DBLP:journals/corr/abs-2105-05541}
Sharma, S., Dey, M., and Sinha, K. (2021).
\newblock Evaluating gender bias in natural language inference.
\newblock {\em CoRR}, abs/2105.05541.

\bibitem[Shin et~al., 2020]{autoprompt}
Shin, T., Razeghi, Y., Logan~IV, R.~L., Wallace, E., and Singh, S. (2020).
\newblock Autoprompt: Eliciting knowledge from language models with automatically generated prompts.
\newblock {\em arXiv preprint arXiv:2010.15980}.

\bibitem[Socher et~al., 2013]{sst2}
Socher, R., Perelygin, A., Wu, J., Chuang, J., Manning, C.~D., Ng, A., and Potts, C. (2013).
\newblock Recursive deep models for semantic compositionality over a sentiment treebank.
\newblock In {\em Proceedings of EMNLP}, pages 1631--1642.

\bibitem[Srivastava et~al., 2022]{bigbench}
Srivastava, A., Rastogi, A., Rao, A., Shoeb, A. A.~M., Abid, A., Fisch, A., Brown, A.~R., Santoro, A., Gupta, A., Garriga{-}Alonso, A., Kluska, A., Lewkowycz, A., Agarwal, A., Power, A., Ray, A., Warstadt, A., Kocurek, A.~W., Safaya, A., Tazarv, A., Xiang, A., Parrish, A., Nie, A., Hussain, A., Askell, A., Dsouza, A., Rahane, A., Iyer, A.~S., Andreassen, A., Santilli, A., Stuhlm{\"{u}}ller, A., Dai, A.~M., La, A., Lampinen, A.~K., Zou, A., Jiang, A., Chen, A., Vuong, A., Gupta, A., Gottardi, A., Norelli, A., Venkatesh, A., Gholamidavoodi, A., Tabassum, A., Menezes, A., Kirubarajan, A., Mullokandov, A., Sabharwal, A., Herrick, A., Efrat, A., Erdem, A., Karakas, A., and et~al. (2022).
\newblock Beyond the imitation game: Quantifying and extrapolating the capabilities of language models.
\newblock {\em CoRR}, abs/2206.04615.

\bibitem[Su et~al., 2023]{DBLP:conf/iclr/SuKWSWX0OZS023}
Su, H., Kasai, J., Wu, C.~H., Shi, W., Wang, T., Xin, J., Zhang, R., Ostendorf, M., Zettlemoyer, L., Smith, N.~A., and Yu, T. (2023).
\newblock Selective annotation makes language models better few-shot learners.
\newblock In {\em The Eleventh International Conference on Learning Representations, {ICLR} 2023, Kigali, Rwanda, May 1-5, 2023}. OpenReview.net.

\bibitem[Thoppilan et~al., 2022]{lamda}
Thoppilan, R., De~Freitas, D., Hall, J., Shazeer, N., Kulshreshtha, A., Cheng, H.-T., Jin, A., Bos, T., Baker, L., Du, Y., et~al. (2022).
\newblock Lamda: Language models for dialog applications.
\newblock {\em arXiv preprint arXiv:2201.08239}.

\bibitem[Touvron et~al., 2023a]{llama}
Touvron, H., Lavril, T., Izacard, G., Martinet, X., Lachaux, M., Lacroix, T., Rozi{\`{e}}re, B., Goyal, N., Hambro, E., Azhar, F., Rodriguez, A., Joulin, A., Grave, E., and Lample, G. (2023a).
\newblock Llama: Open and efficient foundation language models.
\newblock {\em CoRR}, abs/2302.13971.

\bibitem[Touvron et~al., 2023b]{llama2}
Touvron, H., Martin, L., Stone, K., Albert, P., Almahairi, A., Babaei, Y., Bashlykov, N., Batra, S., Bhargava, P., Bhosale, S., Bikel, D., Blecher, L., Canton{-}Ferrer, C., Chen, M., Cucurull, G., Esiobu, D., Fernandes, J., Fu, J., Fu, W., Fuller, B., Gao, C., Goswami, V., Goyal, N., Hartshorn, A., Hosseini, S., Hou, R., Inan, H., Kardas, M., Kerkez, V., Khabsa, M., Kloumann, I., Korenev, A., Koura, P.~S., Lachaux, M., Lavril, T., Lee, J., Liskovich, D., Lu, Y., Mao, Y., Martinet, X., Mihaylov, T., Mishra, P., Molybog, I., Nie, Y., Poulton, A., Reizenstein, J., Rungta, R., Saladi, K., Schelten, A., Silva, R., Smith, E.~M., Subramanian, R., Tan, X.~E., Tang, B., Taylor, R., Williams, A., Kuan, J.~X., Xu, P., Yan, Z., Zarov, I., Zhang, Y., Fan, A., Kambadur, M., Narang, S., Rodriguez, A., Stojnic, R., Edunov, S., and Scialom, T. (2023b).
\newblock Llama 2: Open foundation and fine-tuned chat models.
\newblock {\em CoRR}, abs/2307.09288.

\bibitem[Wang et~al., 2019]{glue}
Wang, A., Singh, A., Michael, J., Hill, F., Levy, O., and Bowman, S.~R. (2019).
\newblock {GLUE:} {A} multi-task benchmark and analysis platform for natural language understanding.
\newblock In {\em 7th International Conference on Learning Representations, {ICLR} 2019, New Orleans, LA, USA, May 6-9, 2019}. OpenReview.net.

\bibitem[Wang et~al., 2024]{wang2024rethinking}
Wang, J.~T., Yang, T., Zou, J., Kwon, Y., and Jia, R. (2024).
\newblock Rethinking data shapley for data selection tasks: Misleads and merits.

\bibitem[Wang et~al., 2023]{sc}
Wang, X., Wei, J., Schuurmans, D., Le, Q.~V., Chi, E.~H., Narang, S., Chowdhery, A., and Zhou, D. (2023).
\newblock Self-consistency improves chain of thought reasoning in language models.
\newblock In {\em The Eleventh International Conference on Learning Representations, {ICLR} 2023, Kigali, Rwanda, May 1-5, 2023}. OpenReview.net.

\bibitem[Wang et~al., 2022]{rationaleaugment}
Wang, X., Wei, J., Schuurmans, D., Le, Q.~V., Chi, E.~H., and Zhou, D. (2022).
\newblock Rationale-augmented ensembles in language models.
\newblock {\em CoRR}, abs/2207.00747.

\bibitem[Wei et~al., 2022]{wei22cot}
Wei, J., Wang, X., Schuurmans, D., Bosma, M., Ichter, B., Xia, F., Chi, E.~H., Le, Q.~V., and Zhou, D. (2022).
\newblock Chain-of-thought prompting elicits reasoning in large language models.
\newblock In Koyejo, S., Mohamed, S., Agarwal, A., Belgrave, D., Cho, K., and Oh, A., editors, {\em Advances in Neural Information Processing Systems 35: Annual Conference on Neural Information Processing Systems 2022, NeurIPS 2022, New Orleans, LA, USA, November 28 - December 9, 2022}.

\bibitem[Wolf et~al., 2019]{huggingface}
Wolf, T., Debut, L., Sanh, V., Chaumond, J., Delangue, C., Moi, A., Cistac, P., Rault, T., Louf, R., Funtowicz, M., and Brew, J. (2019).
\newblock Huggingface's transformers: State-of-the-art natural language processing.
\newblock {\em CoRR}, abs/1910.03771.

\bibitem[Wortsman et~al., 2022]{finetune}
Wortsman, M., Ilharco, G., Kim, J.~W., Li, M., Kornblith, S., Roelofs, R., Lopes, R.~G., Hajishirzi, H., Farhadi, A., Namkoong, H., and Schmidt, L. (2022).
\newblock Robust fine-tuning of zero-shot models.
\newblock In {\em {IEEE/CVF} Conference on Computer Vision and Pattern Recognition, {CVPR} 2022, New Orleans, LA, USA, June 18-24, 2022}, pages 7949--7961. {IEEE}.

\bibitem[Xu et~al., 2024]{DBLP:journals/corr/abs-2403-09963}
Xu, Z., Peng, K., Ding, L., Tao, D., and Lu, X. (2024).
\newblock Take care of your prompt bias! investigating and mitigating prompt bias in factual knowledge extraction.
\newblock {\em CoRR}, abs/2403.09963.

\bibitem[Yang et~al., 2023]{probselect}
Yang, S., Kim, J., Jang, J., Ye, S., Lee, H., and Seo, M. (2023).
\newblock Improving probability-based prompt selection through unified evaluation and analysis.
\newblock {\em CoRR}, abs/2305.14877.

\bibitem[Zelikman et~al., 2022]{star}
Zelikman, E., Wu, Y., Mu, J., and Goodman, N.~D. (2022).
\newblock Star: Bootstrapping reasoning with reasoning.
\newblock In Koyejo, S., Mohamed, S., Agarwal, A., Belgrave, D., Cho, K., and Oh, A., editors, {\em Advances in Neural Information Processing Systems 35: Annual Conference on Neural Information Processing Systems 2022, NeurIPS 2022, New Orleans, LA, USA, November 28 - December 9, 2022}.

\bibitem[Zhang et~al., 2023a]{DBLP:journals/pacmmod/0006SL0P023}
Zhang, J., Sun, Q., Liu, J., Xiong, L., Pei, J., and Ren, K. (2023a).
\newblock Efficient sampling approaches to shapley value approximation.
\newblock {\em Proc. {ACM} Manag. Data}, 1(1):48:1--48:24.

\bibitem[Zhang et~al., 2023b]{DBLP:conf/icde/0006XSL0P023}
Zhang, J., Xia, H., Sun, Q., Liu, J., Xiong, L., Pei, J., and Ren, K. (2023b).
\newblock Dynamic shapley value computation.
\newblock In {\em 39th {IEEE} International Conference on Data Engineering, {ICDE} 2023, Anaheim, CA, USA, April 3-7, 2023}, pages 639--652. {IEEE}.

\bibitem[Zhang et~al., 2023c]{autocot}
Zhang, Z., Zhang, A., Li, M., and Smola, A. (2023c).
\newblock Automatic chain of thought prompting in large language models.
\newblock In {\em The Eleventh International Conference on Learning Representations, {ICLR} 2023, Kigali, Rwanda, May 1-5, 2023}. OpenReview.net.

\bibitem[Zhao et~al., 2021]{calibrate}
Zhao, Z., Wallace, E., Feng, S., Klein, D., and Singh, S. (2021).
\newblock Calibrate before use: Improving few-shot performance of language models.
\newblock In Meila, M. and Zhang, T., editors, {\em Proceedings of the 38th International Conference on Machine Learning, {ICML} 2021, 18-24 July 2021, Virtual Event}, volume 139 of {\em Proceedings of Machine Learning Research}, pages 12697--12706. {PMLR}.

\bibitem[Zhong et~al., 2021]{optiprompt}
Zhong, Z., Friedman, D., and Chen, D. (2021).
\newblock Factual probing is {[MASK]:} learning vs. learning to recall.
\newblock In Toutanova, K., Rumshisky, A., Zettlemoyer, L., Hakkani{-}T{\"{u}}r, D., Beltagy, I., Bethard, S., Cotterell, R., Chakraborty, T., and Zhou, Y., editors, {\em Proceedings of the 2021 Conference of the North American Chapter of the Association for Computational Linguistics: Human Language Technologies, {NAACL-HLT} 2021, Online, June 6-11, 2021}, pages 5017--5033. Association for Computational Linguistics.

\bibitem[Zhou et~al., 2022]{DBLP:journals/ijcv/ZhouYLL22}
Zhou, K., Yang, J., Loy, C.~C., and Liu, Z. (2022).
\newblock Learning to prompt for vision-language models.
\newblock {\em Int. J. Comput. Vis.}, 130(9):2337--2348.

\end{thebibliography}






\clearpage
\appendix

\section{Complete proof}\label{appendix:proofs}
\subsection{Proof of Equation \ref{e6}}\label{e6proof}
\paragraph{Account for changes in accuracy}First, we simplify each prompt into a sub-classifier $h_i$ and the prompt ensemble into a final classifier $f$. Let $f'$ be the model with $h_i$ replaced by $h_i'$. We then consider the change in accuracy on the dataset $\mathcal{V}$:
\begin{align*}
    |\mathcal{U}(f) &- \mathcal{U}(f')|,\\
    \mathcal{U}(f) = \mathcal{ACC}(f) &= \frac{1}{|\mathcal{V}|}\sum_{(x, y) \in \mathcal{V}} \mathbb{I}(f(x) = y) ,
\end{align*}
where $ \mathbb{I} $ is the indicator function that equals 1 if $f(x) = y $ and 0 otherwise.

After replacing the $k$-th sub-classifier $h_k$ with $h_k'$, the new model $f'$ has an accuracy:
\begin{equation*}
    \mathcal{U}(f') = \mathcal{ACC}(f') = \frac{1}{|\mathcal{V}|}\sum_{(x, y) \in \mathcal{V}} \mathbb{I}(f(x) = y).
\end{equation*}

The change in the model output for a data point $x$ when replacing $h_k$ with $h_k'$ is:
\begin{align*}
      |f(x) - f'(x)| &= \left| \frac{1}{N} \sum_{i=1}^{N} h_i(x) - \left( \frac{1}{N} \sum_{i \neq k} h_i(x) + \frac{1}{N} h_k'(x) \right) \right|\\
      &= \left| \frac{1}{N} h_k(x) - \frac{1}{N} h_k'(x) \right|\\
      &= \frac{1}{N} |h_k(x) - h_k'(x)|.
\end{align*}
The accuracy of the dataset changes depending on how many data points have their classification results changed. If for each $x$, changing $\frac{1}{N} |h_k(x) - h_k'(x)|$ causes the model output to change from correct to incorrect or from incorrect to correct, then accuracy will be affected.

Assume that the model output of $m$ data points changes from correct to incorrect or from incorrect to correct after the sub-classifier $h_k$ is replaced by $h_k'$. The impact of such $m$ data points on the accuracy is:
\begin{equation}\label{eq8}
    |\mathcal{U}(f) - \mathcal{U}(f')| = \frac{m}{|\mathcal{V}|}.
\end{equation}
\paragraph{Estimate an upper bound on $m$}
For each data point $x$, the change $|f(x) - f'(x)| = \frac{1}{N}|h_k(x) - h_k'(x)|$, recorded as $\epsilon$, is small when $N$ is large, most data points will not have their classification results changed. Consider the worst-case scenario: data points with a predicted result of $(0.5 \pm \epsilon, 0.5 \mp \epsilon)$ all change from correct prediction to incorrect prediction after modifying a classifier. So we have an upper bound on $m$:
\begin{equation}\label{eq9}
    \frac{m}{|\mathcal{V}|} \leq P(0.5-\epsilon \leq X \leq0.5+\epsilon).
\end{equation}
In the main text, we show that the model's predictions can be mathematically modeled using a Beta distribution $X\sim\mathrm{Be}(\alpha,\beta)$ and according to the central limit theorem, we approximate it to a normal distribution $X \sim N\left(\frac{\alpha}{\alpha+\beta},\sqrt{\frac{\alpha\beta}{(\alpha+\beta)^2(\alpha+\beta+1)}}\right)$, so we get an approximation of this probability:
\begin{equation}\label{eq10}
    P(0.5-\epsilon \leq X \leq0.5+\epsilon)\approx\Phi\left(\frac{0.5+\epsilon-\mu}\sigma\right)-\Phi\left(\frac{0.5-\epsilon-\mu}\sigma\right).
\end{equation}
In order to obtain a polynomial with respect to $\epsilon$, We perform Taylor series expansion on Equation \ref{eq10}
\begin{equation*}
    \Phi(z)\approx\frac12+\frac1{\sqrt{2\pi}}\left(z-\frac{z^3}{3!}+\frac{z^5}{5!}-\cdots\right).
\end{equation*}
For the convenience of calculation, we expand to the 3rd order and substitute  $\frac{0.5 \pm \epsilon - \mu}{\sigma}$ into the Taylor series expansion of $\Phi(z)$, we get:
\begin{align*}
    \Phi\left(\frac{0.5 + \epsilon - \mu}{\sigma}\right) \approx \frac{1}{2} + \frac{1}{\sqrt{2\pi}} \left( \frac{0.5 + \epsilon - \mu}{\sigma} - \frac{\left(\frac{0.5 + \epsilon - \mu}{\sigma}\right)^3}{3!} + \cdots \right),\\
    \Phi\left(\frac{0.5 - \epsilon - \mu}{\sigma}\right) \approx \frac{1}{2} + \frac{1}{\sqrt{2\pi}} \left( \frac{0.5 - \epsilon - \mu}{\sigma} - \frac{\left(\frac{0.5 - \epsilon - \mu}{\sigma}\right)^3}{3!} + \cdots \right).
\end{align*}
Subtracting these two expansions:
\begin{equation*}
    \Delta_\Phi \approx \frac{1}{\sqrt{2\pi}} \left( \frac{0.5 + \epsilon - \mu}{\sigma} - \frac{\left(\frac{0.5 + \epsilon - \mu}{\sigma}\right)^3}{3!} - \left( \frac{0.5 - \epsilon - \mu}{\sigma} - \frac{\left(\frac{0.5 - \epsilon - \mu}{\sigma}\right)^3}{3!} \right) \right).
\end{equation*}
Simplifying each term:
\begin{equation*}
    \frac{0.5 + \epsilon - \mu}{\sigma} - \frac{0.5 - \epsilon - \mu}{\sigma} = \frac{(0.5 + \epsilon - \mu) - (0.5 - \epsilon - \mu)}{\sigma} = \frac{2\epsilon}{\sigma}.
\end{equation*}
And,
\begin{equation*}
    \left(\frac{0.5 + \epsilon - \mu}{\sigma}\right)^3 - \left(\frac{0.5 - \epsilon - \mu}{\sigma}\right)^3.
\end{equation*}
Using the binomial expansion:
\begin{equation*}
    (a + b)^3 - (a - b)^3 = 2ab(a + b) + 2ab(a - b) = 4a^2b,
\end{equation*}
where $a = \frac{0.5 - \mu}{\sigma}$ and $b = \frac{\epsilon}{\sigma}$, thus,
\begin{equation*}
    \left(\frac{0.5 + \epsilon - \mu}{\sigma}\right)^3 - \left(\frac{0.5 - \epsilon - \mu}{\sigma}\right)^3 = 4 \left(\frac{0.5 - \mu}{\sigma}\right)^2 \left(\frac{\epsilon}{\sigma}\right).
\end{equation*}
Therefore,
\begin{align*}
    \Phi\left(\frac{0.5 + \epsilon - \mu}{\sigma}\right) - &\Phi\left(\frac{0.5 - \epsilon - \mu}{\sigma}\right) 
    \approx \frac{1}{\sqrt{2\pi}} \left( \frac{2\epsilon}{\sigma} - \frac{4 \left(\frac{0.5 - \mu}{\sigma}\right)^2 \left(\frac{\epsilon}{\sigma}\right)}{6} \right)\\
    &= \frac{1}{\sqrt{2\pi}} \left( \frac{2\epsilon}{\sigma} - \frac{2 \epsilon \left(\frac{0.5 - \mu}{\sigma}\right)^2}{3\sigma} \right)\\
    &= \frac{1}{\sqrt{2\pi}} \left( \frac{2\epsilon}{\sigma} - \frac{2 \epsilon (0.5 - \mu)^2}{3\sigma^3} \right)\\
    &= \frac{2\epsilon}{\sqrt{2\pi}\sigma} \left( 1 - \frac{(0.5 - \mu)^2}{3\sigma^2} \right).
\end{align*}
So, the cumulative probability can be expressed as a polynomial approximation in terms of $\epsilon$:
\begin{equation*}
    P(0.5 - \epsilon \leq X \leq 0.5 + \epsilon) \approx \frac{2\epsilon}{\sqrt{2\pi}\sigma} \left( 1 - \frac{(0.5 - \mu)^2}{3\sigma^2} \right),
\end{equation*}
Where $\epsilon = \frac{1}{N}|h_k(x) - h_k'(x)|$ and constant $\mu = \frac{\alpha}{\alpha+\beta}$, $\sigma = \sqrt{\frac{\alpha\beta}{(\alpha+\beta)^2(\alpha+\beta+1)}}$.
Substituting this bound into Equation \ref{eq8} and \ref{eq9}, we have:
\begin{equation*}
    |\mathcal{U}(f) - \mathcal{U}(f')| \leq \frac{L}{N}|h_k(x) - h_k'(x)|,
\end{equation*}
where $L = \frac{2}{\sqrt{2\pi}\sigma} \left( 1 - \frac{(0.5 - \mu)^2}{3\sigma^2} \right)$ and $N$ is the number of sub-classifiers.

\subsection{Proof of Theorem \ref{t1}}\label{t1proof}
To prove Theorem 1, we must first establish a lemma:
\begin{lemma}\label{l1}
    For any arbitrary integer $k\in[0,n-2]$,
    \begin{equation*}
        \frac{1}{n}(\frac{1}{\binom{n-1}{k}}+\frac{1}{\binom{n-1}{k+1}})=\frac{1}{(n-1)\binom{n-2}{k}}.
    \end{equation*}
\end{lemma}
\begin{proof}
    \begin{align*}
        \frac{1}{n}(\frac{1}{\binom{n-1}{k}}+\frac{1}{\binom{n-1}{k+1}})& =\frac{k!(n-1-k)!+(k+1)!(n-2-k)!}{n!}  \\
        &=\frac{k!(n-2-k)!(k+1+n-1-k)}{n!} \\
        &=\frac{k!(n-2-k)!}{(n-1)\cdot(n-2)!}\\
        &=\frac{1}{(n-1)\binom{n-2}{k}}.
    \end{align*}
\end{proof}
Then we prove Theorem 1: 
\begin{proof}
\normalfont
\begin{align*}
    &|\mathcal{SV}_i - \mathcal{SV}_j|\\
    =&\frac{1}{n}\left|\sum_{\mathcal{S} \subseteq \mathcal{N} \setminus \{\bm{e}_i\}}{\frac{\mathcal{U} (\mathcal{S} \cup \{ \bm{e}_i\} - \mathcal{U}(\mathcal{S}))}{\binom{n-1}{|\mathcal{S}|}}} - \sum_{\mathcal{S} \subseteq \mathcal{N} \setminus \{\bm{e}_j\}}{\frac{\mathcal{U} (\mathcal{S} \cup \{ \bm{e}_i\} - \mathcal{U}(\mathcal{S}))}{\binom{n-1}{|\mathcal{S}|}}}\right|\\
    =&\frac{1}{n} \left|\sum_{\mathcal{S} \subseteq \mathcal{N} \setminus \{\bm{e}_i, \bm{e}_j\}}{\frac{\mathcal{U}(\mathcal{S} \cup \{\bm{e}_i\} - \mathcal{U}(\mathcal{S} \cup \{\bm{e}_j\})}{\binom{n-1}{|\mathcal{S}|}} + \frac{\mathcal{U}(\mathcal{S} \cup \{\bm{e}_i\} - \mathcal{U}(\mathcal{S} \cup \{\bm{e}_j\}}{\binom{n-1}{|\mathcal{S}|+1}}}\right|\\
    =&\frac{1}{n-1} \left|\sum_{\mathcal{S} \subseteq \mathcal{N} \setminus \{\bm{e}_i, \bm{e}_j\}}{\frac{\mathcal{U}(\mathcal{S} \cup \{\bm{e}_i\} - \mathcal{U}(\mathcal{S} \cup \{\bm{e}_j\})}{\binom{n-2}{|\mathcal{S}|}}}\right|(\textbf{Using Lemma} \ref{l1})\\
    \leq&\frac{1}{n-1} \sum_{\mathcal{S} \subseteq \mathcal{N} \setminus \{\bm{e}_i, \bm{e}_j\}}{\frac{|\mathcal{U}(\mathcal{S} \cup \{\bm{e}_i\} - \mathcal{U}(\mathcal{S} \cup \{\bm{e}_j\})|}{\binom{n-2}{|\mathcal{S}|}}}\\
    \leq& \frac{1}{n-1} \sum_{\mathcal{S} \subseteq \mathcal{N} \setminus \{\bm{e}_i, \bm{e}_j\}}{\frac{L \cdot ||\bm{e}_i - \bm{e}_j||}{\binom{n-2}{|\mathcal{S}|}}}\\
    \leq& L \cdot ||\bm{e}_i - \bm{e}_j||\frac{1}{n-1} \sum_{\mathcal{S} \subseteq \mathcal{N} \setminus \{\bm{e}_i, \bm{e}_j\}}{\frac{1}{\binom{n-2}{|\mathcal{S}|}}}\\
\end{align*}
For the coefficient:
\begin{align*}   
    \frac{1}{n-1} \sum_{\mathcal{S} \subseteq \mathcal{N} \setminus \{\bm{e}_i, \bm{e}_j\}}{\frac{1}{\binom{n-2}{|\mathcal{S}|}}}
    &= \frac{1}{n-1} \sum_{k=0}^{n-2}{\binom{n-2}{k}}\frac{1}{\binom{n-2}{k}}\\
    &= \frac{1}{n-1} \cdot (n-1)\\
    &=1.
\end{align*}
So we have:
\begin{equation*}
    |\mathcal{SV}_i - \mathcal{SV}_j| \leq L \cdot ||\bm{e}_i - \bm{e}_j||.
\end{equation*}
\end{proof}

\section{Additional plots}


\begin{figure}[h]
  \centering
  \begin{subfigure}{0.4\textwidth}
    \includegraphics[width=\textwidth]{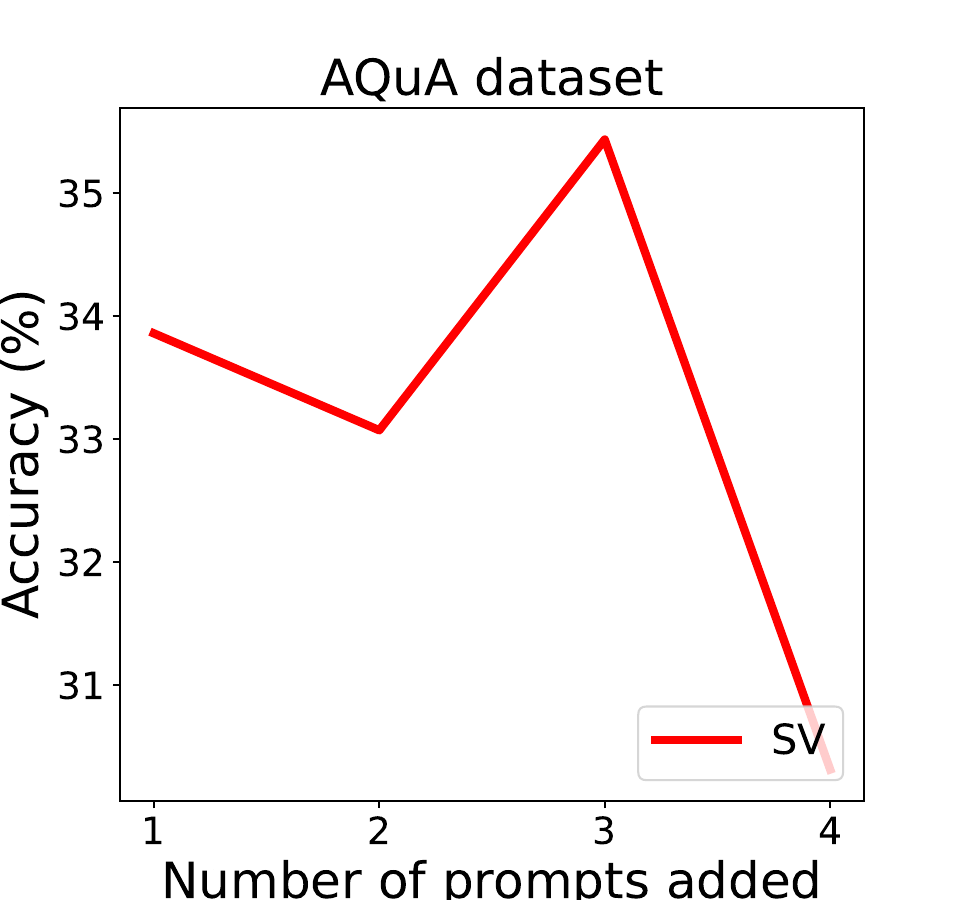}
  \end{subfigure}
  \hfill
  \begin{subfigure}{0.4\textwidth}
    \includegraphics[width=\textwidth]{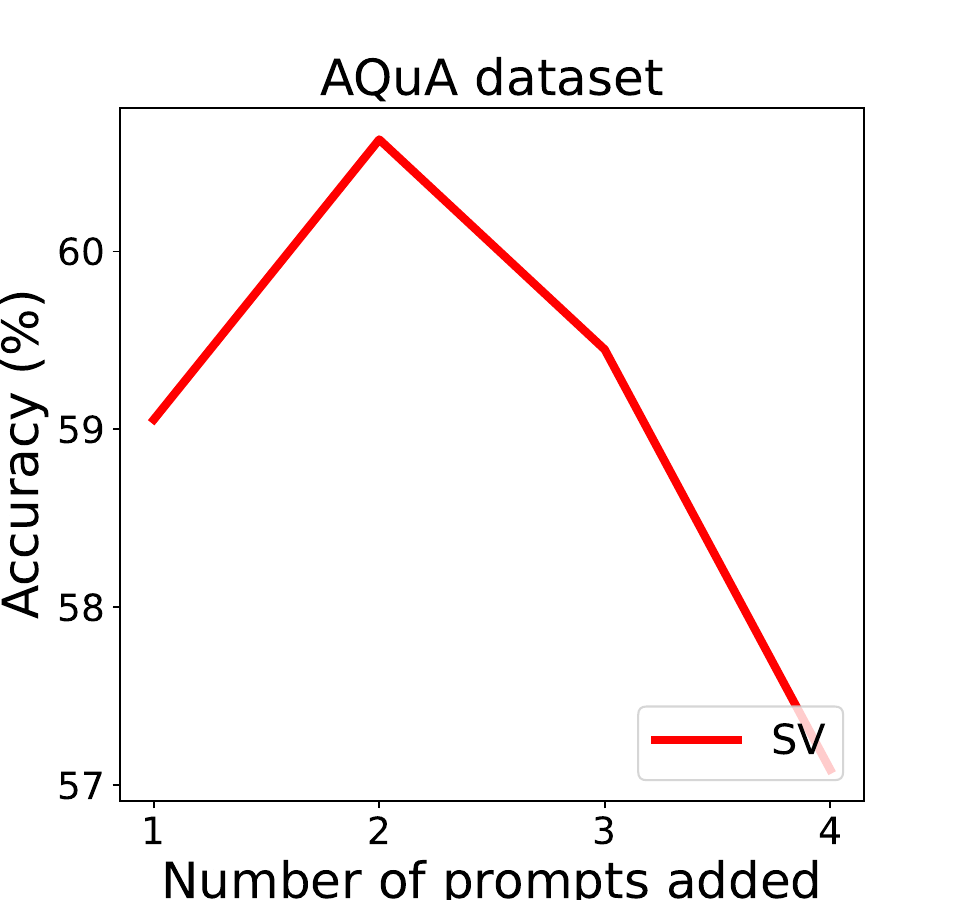}
  \end{subfigure}
  \caption{Results for AQuA using Manual prompt (Manual-CoT but without rationale) and Auto-CoT.}
\end{figure}

\begin{figure}[h]
  \centering
  \begin{subfigure}{0.4\textwidth}
    \includegraphics[width=\textwidth]{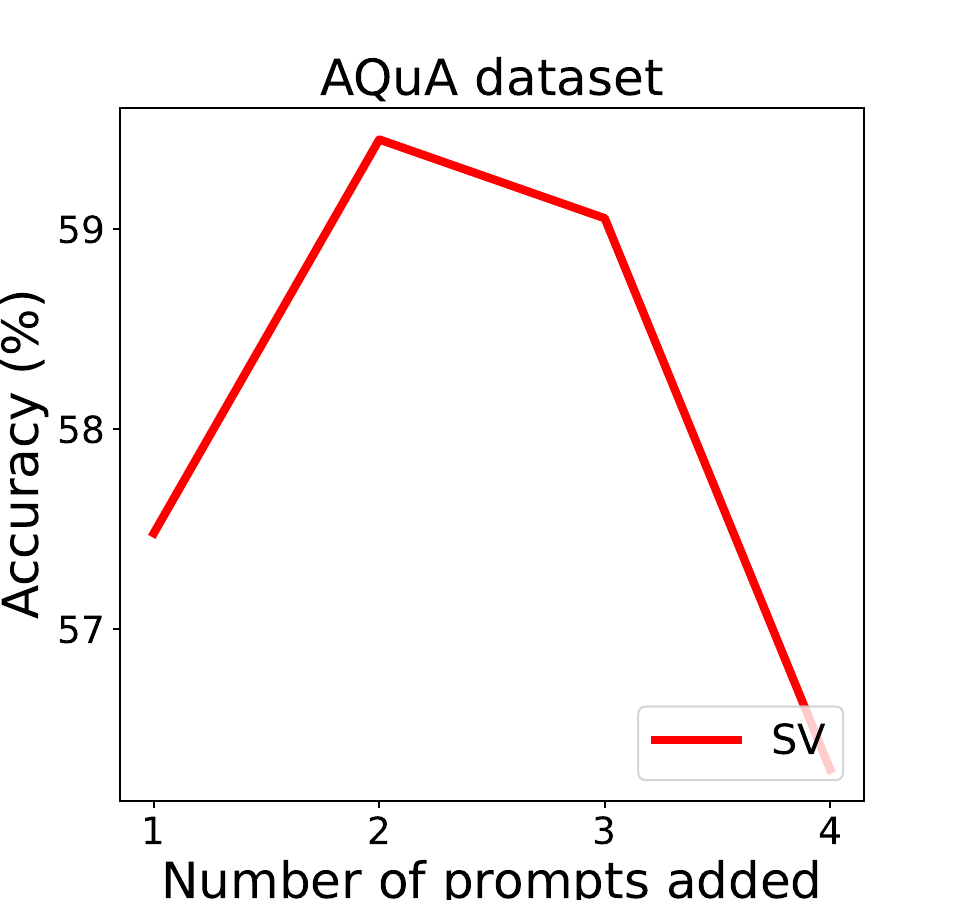}
  \end{subfigure}
  \hfill
  \begin{subfigure}{0.4\textwidth}
    \includegraphics[width=\textwidth]{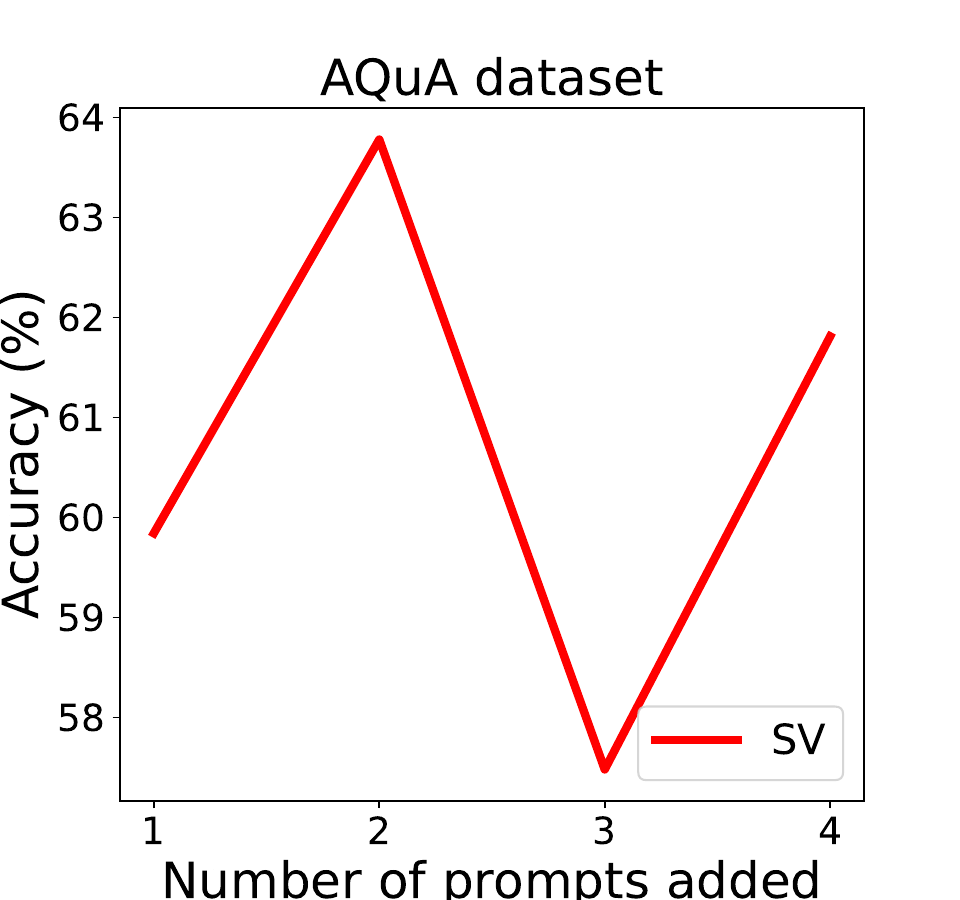}
  \end{subfigure}
  \caption{Results for AQuA using Active-CoT and Complex-CoT.}
\end{figure}

\begin{figure}[h]
    \centering
    \includegraphics[width=0.4\linewidth]{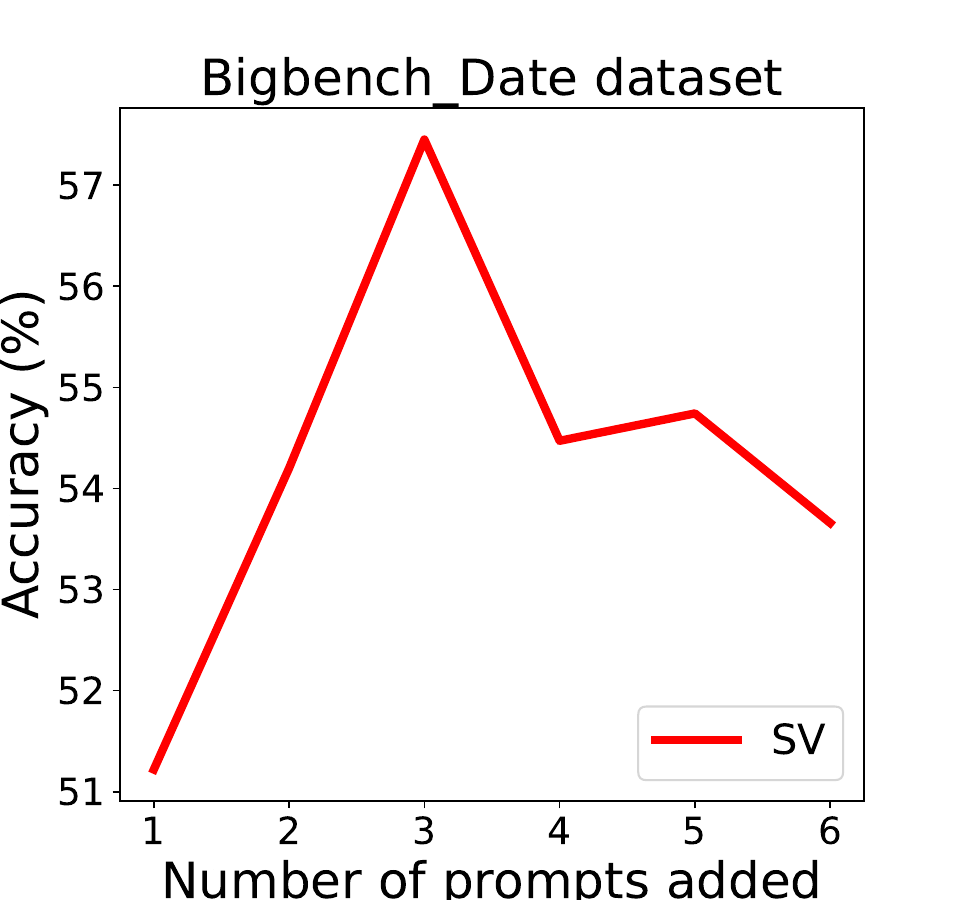}
    \caption{Result for Date using Manual prompt (Manual-CoT but without rationale).}
\end{figure}

    
    




\newpage
\section{Resources}\label{appendix:resources}
The dataset used can be obtained through the following URL.
\begin{itemize}
    \item SST2-sentiment-analysis:
    
    \url{https://github.com/YJiangcm/SST-2-sentiment-analysis}
    
    \item AQUA-RAT:
    
    \url{https://github.com/google-deepmind/AQuA}
    
    
    
    \item Date-understanding:
    
    \url{https://github.com/google/BIG-bench}
\end{itemize}

The models and checkpoints used can be obtained through the following URL.
\begin{itemize}
    \item BERT-base:
    
    \url{https://huggingface.co/google-bert/bert-base-uncased}
    
    \item GPT-3.5-turbo:
    
    \url{https://openai.com/api/}
\end{itemize}

We used one NVIDIA RTX 3090 GPU. Note that experiments only require sufficient memory to load the BERT model, and calculating the Shapley value and training machine learning models can be done on a common CPU. 

\end{document}